\documentclass[10pt,a4paper,twocolumn]{article}

\usepackage{dcpreprint}

\usepackage{textcomp}
\usepackage{eso-pic}
\newcommand{\ieeesubmissionnotice}{%
  \begingroup
  \setlength{\fboxsep}{7pt}%
  \setlength{\fboxrule}{0.4pt}%
  \fbox{\begin{minipage}{\dimexpr\textwidth-2\fboxsep-2\fboxrule\relax}
    \footnotesize\raggedright
    This work has been submitted to the IEEE for possible publication. Copyright may be transferred without notice, after which this version may no longer be accessible.
  \end{minipage}}%
  \endgroup
}
\fancypagestyle{dcfirst}{%
  \fancyhf{}%
  \fancyfoot[C]{\footnotesize\thepage}%
  }

\allowdisplaybreaks

\graphicspath{{figures/}{./figures/}{./}}

\newcommand{\figplaceholder}[3]{%
  \setlength{\fboxsep}{0pt}%
  \fcolorbox{black!35}{black!4}{%
    \begin{minipage}[c][#2][c]{\dimexpr#1-2\fboxrule\relax}%
      \centering\color{black!62}%
      {\footnotesize\bfseries [\,figure placeholder\,]}\\[3pt]
      {\scriptsize\ttfamily #3}%
    \end{minipage}}%
}
\newcommand{\incfig}[4][\columnwidth]{%
  \IfFileExists{#2}{\includegraphics[width=#1]{#2}}%
                   {\message{^^J*** MISSING FIGURE FILE: #2 -- grey placeholder box used ***^^J}%
                    \figplaceholder{#1}{#3}{#4}}%
}

\newtheorem{theorem}{Theorem}
\newtheorem{lemma}{Lemma}
\newtheorem{corollary}{Corollary}

\theoremstyle{definition}

\newtheorem{remark}{Remark}

\newcommand{\R}{\mathbb{R}}
\newcommand{\Sym}{\operatorname{Sym}}
\newcommand{\col}{\operatorname{col}}
\newcommand{\diag}{\operatorname{diag}}
\newcommand{\tp}{\top}
\newcommand{\lmin}{\ell^{-}}
\newcommand{\lmax}{\ell^{+}}
\newcommand{\Lm}{L^{-}}
\newcommand{\Lp}{L^{+}}

\newcommand{\lbar}{\bar\ell}
\newcommand{\taub}{\bar\tau}
\newcommand{\Xibf}{\Xi}
\newcommand{\Ac}{\mathcal{A}}
\newcommand{\Fc}[1]{\mathcal{F}_{#1}}
\newcommand{\Rc}[1]{\mathcal{R}_{#1}}
\newcommand{\Sc}{\mathcal{S}}

\newcommand{\Tc}{\mathcal{T}}

\hypersetup{
  pdftitle={Length-independent stability and L2-gain certificates for chains of delayed neural modules},
  pdfauthor={\"{O}mer Faruk \.{I}\c{s}canl\i},
  pdfkeywords={delayed neural networks, interconnected systems, L2 gain, linear matrix inequalities, Lyapunov-Krasovskii functionals, stability analysis, time-varying delays}
}

\title{Length-independent stability and $L_{2}$-gain certificates for chains of delayed neural modules}
\author{{\"O}mer Faruk {\.I}{\c{s}}canl{\i}}
\orcid{0009-0009-6006-7120}
\affil{Department of Computer Engineering, Istanbul University-Cerrahpa{\c{s}}a, Istanbul, T{\"u}rkiye}
\email{omerfarukiscanli@gmail.com}
\runningtitle{Length-independent certificates for chains of delayed neural modules}
\runningauthor{{\"O}.F. {\.I}{\c{s}}canl{\i}}

\setabstract{%
Networks of repeated neural modules with delayed neighbour coupling need
stability guarantees for every number of modules, but delay-dependent
criteria certify one network at a time and grow with its size. This paper
derives stability and $L_{2}$-gain certificates of length-independent size
for open chains and rings of modules of two alternating types with
slope-restricted activations. A local and a transmission delay act
independently and are only bounded and continuous. A Lyapunov--Krasovskii
functional with fixed integration limits and module-local variables that
repeat with the module type reduces the certificate of every chain to one
cell of two modules. Local conditions and a single linear matrix inequality
of order 96 on this cell, or 104 for the gain, certify global asymptotic
stability and a common $L_{2}$-gain bound for all open chains and all rings
with an even number of at least four modules. Up to a margin, the stability
inequality is also necessary for the repeated variables to certify every
length. All uniform certificates, including certified coupling regions,
are verified in exact rational arithmetic. In the numerical study, the
certified coupling reaches 92\% of an outer bound given by unstable members
of the model class for short delays and 79\% at the nominal delay bounds.
For 16 modules at unit coupling strength, certificates of the same
functional computed for that length alone raise the certified delay scale
by 2\% to 7\% in four cases. A single stability problem with 1298 scalar
variables thus replaces per-length certification of long chains at little
cost.}


\begin{document}

\AddToShipoutPictureFG*{%
  \AtPageUpperLeft{%
    \put(\LenToUnit{18mm},-\LenToUnit{22pt}){%
      \makebox(0,0)[lt]{\ieeesubmissionnotice}%
    }%
  }%
}
\maketitle

\section{Introduction}
\label{sec:intro}
Arrays of coupled neural networks consist of identical networks that
interact through coupling matrices \citep{zheng2018stochastic}. Spatially
interconnected systems likewise couple identical subsystems with their
neighbours \citep{wang2020asymptotic}. In hardware implementations of
neural networks, delays arise from the finite switching speed of amplifiers
and from communication times, and they can cause oscillation and
instability \citep{zhang2003globally}. This paper considers
chains and rings of delayed neural modules of two alternating types. One
delay acts in the local feedback of each module and another in the signals
exchanged between neighbours. Both delays are bounded and continuous,
whereas their rates of change and their relation are unknown. If the
number of modules is a design parameter or is not known in advance,
stability and an $L_{2}$-gain bound are needed for every chain length
rather than for one.

Delay-dependent criteria for delayed neural networks are typically derived
from Lyapunov--Krasovskii functionals (LKFs) and stated as linear matrix
inequalities (LMIs) \citep{wang2014review}. Richer functionals enlarge the
admissible delay bounds but add decision variables \citep{chen2019improved}.
Criteria are available both for delays with a bounded rate of change and
for delays without rate information \citep{zhai2023improved}. Each of these
criteria certifies one network, and its variables are matrices of the
network dimension. For chains of the modules considered here, the
order-one conditions of \citet{zhai2025hierarchical} reach 155648 scalar
variables and an LMI of order 608 for 16 modules (Section~\ref{ssec:zhai}),
and a certificate for one length implies nothing for another. Delay-independent criteria constrain the connection matrices
independently of the delay \citep{arik2002analysis}, so they give the same
answer for short and for long delays.

For linear spatially interconnected systems with delays, conditions that
exploit the repetition of identical units are known
\citep{wang2020asymptotic}. Synchronization conditions for arrays of
identical neural networks exploit it as well \citep{zheng2018stochastic}
(Section~\ref{ssec:rw-identical}). They do not cover chains in which
slope-restricted activations, two delays without rate bounds and a
coupling through delayed activations occur together. A further question
concerns the assessment of a certificate. Conservatism is usually measured
by comparing the admissible delay bounds of different criteria
\citep{zhai2023improved}. Such comparisons rank criteria, but they do not
show how far a certified set is from what any certificate for the model
class could achieve.

This paper derives stability and $L_{2}$-gain certificates whose size does
not depend on the number of modules. A rate-free LKF with fixed
integration limits gives a network LMI that depends on the delay bounds
but not on the rates of change of the delays. Its variables are restricted
to module-local blocks that repeat with the module type. In the coordinates
of cells of two modules, the certificate matrix of every chain is then
block tridiagonal with the same two cell blocks, and open chains add only
a boundary term of favourable sign. One LMI on a single cell therefore
certifies every open chain and every ring with an even number of at least
four modules. Up to a margin, this LMI is equivalent to the requirement
that the repeated variables certify every chain length, so the reduction to
one cell adds no conservatism beyond the restriction of the variables.
Since the certificate matrix is convex in the coupling strength,
certificates at two coupling values that share the type variables cover the
interval between them. Every certificate for all lengths is verified in
exact rational arithmetic, so it holds independently of solver tolerances.
Unstable linear members of the model class bound from above the largest
coupling up to which every member is stable, and lower bounds on the
worst-case gain show how far the certified gains can be from it.

The main contributions are as follows.
\begin{enumerate}
\item A size-uniform certificate (Theorem~\ref{thm:uniform}) for global
asymptotic stability and a common $L_{2}$-gain bound of all open chains
with $N\ge 1$ modules and all rings with even $N\ge 4$, under two
independent delays without rate bounds. Besides local conditions, it
consists of one LMI of order 96, or 104 for the gain, and it extends to
module-wise delays (Remark~\ref{rem:modulewise}) and to constant inputs
(Remark~\ref{rem:input}).
\item An exactness result for the cell LMI (Lemma~\ref{lem:exact}) and
certified regions over intervals of coupling strength and delay scale
(Corollary~\ref{cor:region}).
\item An assessment against outer bounds and against certificates
computed for single lengths. For short delays, the certified coupling
reaches 92\% of an outer bound obtained from unstable members of the
class, and 79\% at the nominal delay bounds. For three instances at the
nominal bounds, the certified gains exceed lower bounds on the worst-case
gain by at most 50\%. For 16 modules and unit coupling strength,
certificates of the same functional computed for that length alone raise
the certified delay scale by 2\% to 7\% for four sets of type data,
whereas the uniform stability problem keeps 1298 scalar variables for
every length.
\end{enumerate}

Section~\ref{sec:related} reviews related work, and
Section~\ref{sec:problem} states the problem. Section~\ref{sec:method}
derives the certificates and the bounds used to assess them,
Section~\ref{sec:results} reports the numerical results, and
Section~\ref{sec:conclusion} concludes. The Appendix contains the proofs.

\section{Related work}
\label{sec:related}

\subsection{Delay-dependent criteria}
\label{ssec:rw-lkf}
The conservatism of an LKF criterion stems from the functional and from
the bounds on the integral terms of its derivative \citep{wang2014review}.
The Wirtinger-based inequality \citep{seuret2013wirtinger} tightens Jensen's
bound on these terms. Auxiliary-function-based inequalities contain it as a
special case and tighten it further \citep{park2015auxiliary}. The
reciprocally convex combination \citep{park2011reciprocally} joins the
bounds on the two parts of the delay interval. High-degree reciprocally
convex inequalities include earlier ones as special cases and sharpen this
step \citep{ying2026highdegree}. Criteria whose functionals have
integration limits that depend on $\tau(t)$ assume a bound on $\dot\tau$,
as in \citet{chen2019improved}. Criteria for delays without rate information
omit the terms with such limits \citep{zhai2023improved}. For this case,
hierarchical functionals with multiple-integral states admit integral
inequalities of higher order \citep{zhai2025hierarchical}. Slope bounds of
either sign cover activations that are neither monotonic nor bounded
\citep{zhu2008stability}, such as the SiLU of Lemma~\ref{lem:silu}. The
functional of Section~\ref{ssec:lkf} belongs to this family. It has fixed
integration limits, uses the Wirtinger and reciprocally convex bounds and
includes sector terms on the differences between activations at two
instants, as in \citet{zhai2025hierarchical}.

\subsection{Delay-independent criteria}
\label{ssec:rw-indep}
The M-matrix criteria of \citet{cao2003global} compare the self-feedback
rates with the absolute values of the weights and require delay
derivatives that are nonpositive or bounded by a constant below one. The
M-matrix test of \citet{zhang2003globally} requires only bounded delays and
yields global exponential stability. Both depend on the weights only
through their absolute values, so they cannot exploit the sign structure
of a coupling. For a chain, the positive vector of the test of
\citet{zhang2003globally} can be chosen periodic over the cells, and one
test then covers every length (Section~\ref{ssec:assess}).
Section~\ref{sec:results} uses it as a baseline.

\subsection{Networks of identical subsystems}
\label{ssec:rw-identical}
For arrays of identical neural networks with diffusive coupling,
Kronecker-product identities express the synchronization conditions
through LMIs of the size of one network. One LMI is imposed for each pair
of networks, and its coefficients contain the number of networks and the
coupling weights \citep{zheng2018stochastic}. Spatially interconnected
systems describe identical subsystems that interact with their neighbours
through spatial shift operators. For linear subsystems on infinite or
periodic arrays, LMIs of the size of one subsystem give sufficient
conditions for analysis and distributed control with an induced-norm
performance measure \citep{dandrea2003distributed}.
For linear subsystems with a constant state delay, LMIs in the matrices of
one subsystem certify well-posedness, asymptotic stability and
contractiveness \citep{wang2020asymptotic}. A related LMI condition
addresses distributed control with a time-varying delay whose derivative
is bounded by a constant below one \citep{wang2023distributed}. The cell
LMI of Section~\ref{ssec:uniform} is likewise posed on one subsystem, here
a cell of two modules. It covers open chains of every length and rings
with slope-restricted activations and two delays without rate bounds.

\subsection{Quadratic constraints and contraction}
\label{ssec:rw-qc}
Quadratic constraints describe the activations in LMI conditions for
feedback systems with neural network controllers, and integral quadratic
constraints extend them to plants with time delays and refine them with
the slope restrictions of the activations \citep{yin2022stability}. A
complete set of quadratic constraints bounds the repeated ReLU as tightly
as quadratic forms allow and can yield less conservative stability and
performance bounds for recurrent networks than earlier sets
\citep{noori2026complete}. Contraction of recurrent networks in weighted
$\ell_{1}$ and $\ell_{\infty}$ norms is certified by linear programs or by
a Hurwitz condition on the Metzler majorant of the synaptic matrix
\citep{davydov2022non}. Recurrent equilibrium networks are parametrized so
that every member is contracting and satisfies prescribed incremental
integral quadratic constraints \citep{revay2024ren}. The neural networks in
these works have no internal delays, and each analysis concerns one
network of fixed dimension. The sector terms of Section~\ref{ssec:lkf} are quadratic
constraints of this kind. They act at three instants and on the
differences between them, and an LKF accounts for the delays.

\section{Problem formulation}
\label{sec:problem}

\emph{Notation:} $\Sym(M):=M+M^{\tp}$, $\col(\cdot)$ stacks its
arguments, $\diag(\cdot)$ is block diagonal, $\otimes$ is the Kronecker
product, $\star$ marks blocks implied by symmetry, and $\|\cdot\|_{2}$ is
the norm of $L_{2}[0,\infty)$. For square blocks $H_{0}=H_{0}^{\tp}$ and
$H_{1}$, $\Tc_{K}(H_{0},H_{1})$ is block tridiagonal with $K$ diagonal
blocks $H_{0}$, superdiagonal blocks $H_{1}$ and subdiagonal blocks
$H_{1}^{\tp}$, and $\mathcal{C}_{K}(H_{0},H_{1})$ adds $H_{1}$ to its block
$(K,1)$ and $H_{1}^{\tp}$ to its block $(1,K)$.

\subsection{Chains of delayed neural modules}
\label{ssec:model}
A network consists of $N$ modules with two states each. Module $i$ has a
type $t(i)\in\{1,2\}$, and neighbouring modules have different types.
With $y=\col(y^{(1)},\dots,y^{(N)})\in\R^{n}$ and $n=2N$, the network is
\begin{equation}
\label{eq:model}
\begin{aligned}
\dot y(t)={}&-Dy(t)+Ag\bigl(y(t)\bigr)+B_{1}g\bigl(y(t-\tau_{1}(t))\bigr)\\
&+B_{2}g\bigl(y(t-\tau_{2}(t))\bigr)+w(t),\qquad z(t)=y(t),
\end{aligned}
\end{equation}
where $g(y)=\col(g_{1}(y_{1}),\dots,g_{n}(y_{n}))$, $w$ is an exogenous
input and $z$ is the performance output. The matrices $D=\diag(D_{t(i)})\succ 0$,
$A=\diag(A_{t(i)})$ and $B_{1}=\diag(B_{1,t(i)})$ are module-local, with
the data $D_{t}$ (diagonal), $A_{t}$ and $B_{1,t}\in\R^{2\times 2}$ of
type $t$. The matrix $B_{2}$ couples neighbours. Its $2\times 2$ block
$(i,i+1)$ is $F_{f}$, its block $(i+1,i)$ is $F_{b}$, and all its other
blocks vanish. An open chain has the neighbour pairs $(i,i+1)$,
$i=1,\dots,N-1$, and a ring with even $N\ge 4$ has the additional pair
$(N,1)$. Thus $\tau_{1}$ delays the local feedback and $\tau_{2}$ the
transmission between neighbours (Fig.~\ref{fig:chain}). For fixed type
data and couplings $(F_{f},F_{b})$, the family $\mathcal{N}$ consists of
all open chains with $N\ge 1$ modules that start with either type and of
all rings with even $N\ge 4$. Section~\ref{sec:results} uses
$(F_{f},F_{b})=(\kappa\hat G,-\kappa\hat G^{\tp})$ with a coupling
strength $\kappa\ge 0$ and a fixed block $\hat G$, and the pattern
$(\hat G,-\hat G)$.

The analysis rests on three assumptions.

\emph{(H1)} $g_{i}(0)=0$ and $\lmin\le(g_{i}(a)-g_{i}(b))/(a-b)\le\lmax$
for all $a\ne b$ and $i=1,\dots,n$, where $\lmin<\lmax$. We write $\Lm=\lmin I$, $\Lp=\lmax I$ and
$\lbar=\max\{|\lmin|,|\lmax|\}$.

\emph{(H2)} For $k=1,2$, $\tau_{k}:[0,\infty)\to[0,\taub_{k}]$ is
continuous. Neither differentiability nor a bound on the rate of change
nor a relation between $\tau_{1}$ and $\tau_{2}$ is assumed.

\emph{(H3)} $y(\theta)=\phi(\theta)$ for $\theta\in[-h,0]$, where
$h=\max_{k}\taub_{k}$ and $\phi\in C([-h,0],\R^{n})$ is arbitrary.

\begin{lemma}
\label{lem:silu}
The SiLU $f(s)=s/(1+e^{-s})$ \citep{elfwing2018sigmoid} satisfies (H1),
with $g_{i}=f$, for $\lmin=-1/10$ and $\lmax=11/10$.
\end{lemma}

\begin{proof}
With $\sigma(s)=1/(1+e^{-s})$, the derivative
$f'(s)=\sigma(s)[1+s(1-\sigma(s))]$ satisfies $f'(s)+f'(-s)=1$ and tends
to $0$ and $1$ as $s\to-\infty$ and $s\to\infty$. Since
$f''(s)=\sigma(s)(1-\sigma(s))(2-s\tanh(s/2))$ and $s\tanh(s/2)$ is even
and increasing for $s>0$, $f'$ attains its maximum at the root $r>0$ of
$r\tanh(r/2)=2$ and its minimum at $-r$. From $\sigma(r)=(r+2)/(2r)$
follows $f'(r)=1/2+r/4$. The bound $r<12/5$ is equivalent to
$\tanh(6/5)>5/6$, that is, to $e^{12/5}>11$, which holds because
$\sum_{k=0}^{8}(12/5)^{k}/k!>11$. Thus $\sup f'<11/10$ and
$\inf f'=1-\sup f'>-1/10$, and the mean value theorem completes the
proof.
\end{proof}

Only (H1) enters the analysis, so the results hold for any activations
with slopes in $[\lmin,\lmax]$, which may differ between neurons. With
$d\times d$ in place of $2\times 2$ blocks, the analysis also holds for $d$
states per module, with cell LMIs of the orders $48d$ and $52d$.

For given type data and couplings, we seek conditions whose size does
not depend on $N$ and which guarantee for every network in $\mathcal{N}$
(a) global asymptotic stability of the origin for $w\equiv 0$ under
(H1)--(H3) and (b) $\|z\|_{2}\le\gamma\|w\|_{2}$ for zero history, with
one $\gamma$ for the whole family. We further seek (c) sets of coupling
strengths and delay bounds over which (a) holds and (d) upper bounds on
the largest coupling strength up to which every member of the class
defined by (H1)--(H3) is stable.

\begin{figure}[!t]
\centering
\incfig[\columnwidth]{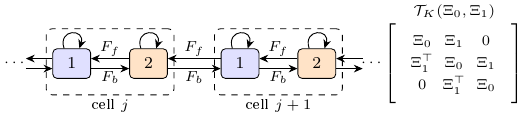}{2cm}{figures/fig\_chain.pdf}
\caption{Open chain of modules of types 1 and 2. Loops mark the local
feedback through $A_{t}$ and $B_{1,t}$, arrows the delayed couplings
$F_{f}$ and $F_{b}$, and dashed boxes the cells. In cell order, the
certificate matrix is block tridiagonal with the cell blocks $\Xibf_{0}$
and $\Xibf_{1}$ (Lemma~\ref{lem:cell}).}
\label{fig:chain}
\end{figure}

\section{Methodology}
\label{sec:method}

\subsection{Rate-free network certificate}
\label{ssec:lkf}
Fix a network in $\mathcal{N}$ and set $\tau_{0}\equiv 0$. For
$j=0,1,2$ and $k=1,2$, let $y_{j}=y(t-\tau_{j}(t))$, $g_{j}=g(y_{j})$ and
$\bar y_{k}=y(t-\taub_{k})$, and let $\nu_{k}$ and $\omega_{k}$ be the
averages of $y$ over $[t-\tau_{k}(t),t]$ and $[t-\taub_{k},t-\tau_{k}(t)]$,
set to $y(t)$ and $\bar y_{k}$ when these intervals are empty. The
augmented vector
\begin{equation}
\label{eq:xi}
\xi=\col(y_{0},y_{1},y_{2},\bar y_{1},\bar y_{2},g_{0},g_{1},g_{2},\nu_{1},\nu_{2},\omega_{1},\omega_{2})
\end{equation}
has the block selectors $e_{y,j},e_{\bar y,k},e_{g,j},e_{\nu,k},
e_{\omega,k}\in\R^{n\times 12n}$, so that, e.g., $y(t)=e_{y,0}\xi(t)$,
and \eqref{eq:model} reads $\dot y=\Ac\xi+w$ with
\begin{equation}
\label{eq:A}
\Ac=-De_{y,0}+Ae_{g,0}+B_{1}e_{g,1}+B_{2}e_{g,2}.
\end{equation}
We use the Lyapunov--Krasovskii functional
\begin{equation}
\label{eq:lkf}
\begin{aligned}
V(t)={}&y^{\tp}(t)Py(t)+2\sum_{i=1}^{n}\pi_{i}\int_{0}^{y_{i}(t)}g_{i}(s)\,ds\\
&+\sum_{k=1}^{2}\int_{t-\taub_{k}}^{t}y^{\tp}(s)S_{k}y(s)\,ds\\
&+\sum_{k=1}^{2}\taub_{k}\int_{-\taub_{k}}^{0}\int_{t+\theta}^{t}\dot y^{\tp}(s)R_{k}\dot y(s)\,ds\,d\theta
\end{aligned}
\end{equation}
with $P=P^{\tp}$, $\Pi=\diag(\pi_{1},\dots,\pi_{n})\succeq 0$,
$S_{k}\succeq 0$ and $R_{k}\succeq 0$. No integration limit in
\eqref{eq:lkf} depends on $\tau_{k}(t)$. The derivative of $V$ therefore
involves the delays only through the delayed states in $\dot y$, and its
bound below uses only the values $\tau_{k}(t)$, whereas a term
$\int_{t-\tau_{k}(t)}^{t}(\cdot)\,ds$ would introduce $\dot\tau_{k}$.

Let $X_{k}\in\R^{2n\times 2n}$, $\hat R_{k}=\diag(R_{k},3R_{k})$,
$\Rc{k}=\bigl[\begin{smallmatrix}\hat R_{k}&X_{k}\\ \star&\hat R_{k}
\end{smallmatrix}\bigr]$ and
\begin{equation}
\label{eq:Fk}
\begin{aligned}
\Fc{k}=\col\bigl(&e_{y,0}-e_{y,k},\;e_{y,0}+e_{y,k}-2e_{\nu,k},\\
&e_{y,k}-e_{\bar y,k},\;e_{y,k}+e_{\bar y,k}-2e_{\omega,k}\bigr).
\end{aligned}
\end{equation}
For $Y,G\in\R^{n\times 12n}$ and a diagonal $T\succeq 0$, let
\begin{equation}
\label{eq:sector}
\Sc(Y,G;T)=\tfrac{1}{2}\Sym\bigl((G-\Lm Y)^{\tp}T(\Lp Y-G)\bigr).
\end{equation}
By (H1), $\xi^{\tp}\Sc(Y,G;T)\xi\ge 0$ whenever $G\xi=g(Y\xi)$, and also when $Y\xi$ and $G\xi$ are
differences of two such pairs. We use \eqref{eq:sector} at the three
instants $t-\tau_{j}(t)$ and, as in \citet{zhai2025hierarchical}, for the
differences between each pair of them. With $W=e_{y,0}^{\tp}P+
e_{g,0}^{\tp}\Pi$, $\tilde R=\taub_{1}^{2}R_{1}+\taub_{2}^{2}R_{2}$,
diagonal multipliers $T_{j}\succeq 0$ and $T_{ab}\succeq 0$, and
\begin{equation}
\label{eq:L}
\begin{aligned}
L={}&\sum_{k=1}^{2}\bigl(e_{y,0}^{\tp}S_{k}e_{y,0}-e_{\bar y,k}^{\tp}
S_{k}e_{\bar y,k}-\Fc{k}^{\tp}\Rc{k}\Fc{k}\bigr)\\
&+\sum_{j=0}^{2}\Sc(e_{y,j},e_{g,j};T_{j})\\
&+\sum_{0\le a<b\le 2}\Sc(e_{y,a}-e_{y,b},e_{g,a}-e_{g,b};T_{ab}),
\end{aligned}
\end{equation}
the certificate matrices are
\begin{equation}
\label{eq:Xi}
\begin{aligned}
\Xibf&=L+\Sym(W\Ac)+\Ac^{\tp}\tilde R\Ac,\\
\tilde\Xibf(\gamma)&=\begin{bmatrix}\Xibf+e_{y,0}^{\tp}e_{y,0}&
W+\Ac^{\tp}\tilde R\\ \star&\tilde R-\gamma^{2}I\end{bmatrix}.
\end{aligned}
\end{equation}
Setting $T_{ab}=0$ removes the difference terms from \eqref{eq:L}, and
Section~\ref{sec:results} uses this case to measure their effect.

\begin{lemma}
\label{lem:vdot}
Let (H1) and (H2) hold, and let $\Pi$, $S_{k}$, $R_{k}$, $T_{j}$,
$T_{ab}$ and $\Rc{k}$ be positive semidefinite. For almost every $t$ such
that $y$ is absolutely continuous on $[t-h,t]$ with a square-integrable
derivative,
\begin{equation}
\label{eq:vdot}
\dot V\le\xi^{\tp}L\xi+2\xi^{\tp}W\dot y+\dot y^{\tp}\tilde R\dot y.
\end{equation}
Hence $\dot V\le\xi^{\tp}\Xibf\xi$ if $w=0$, and
$\dot V+z^{\tp}z-\gamma^{2}w^{\tp}w\le\tilde\xi^{\tp}\tilde\Xibf(\gamma)
\tilde\xi$ with $\tilde\xi=\col(\xi,w)$.
\end{lemma}

\begin{theorem}
\label{thm:network}
Let (H1)--(H3) hold, and let $P=P^{\tp}$, $\Pi\succeq 0$,
$S_{k}\succeq 0$, $R_{k}\succeq 0$, $X_{k}$ and diagonal
$T_{j}\succeq 0$, $T_{ab}\succeq 0$ satisfy
\begin{equation}
\label{eq:C}
\text{(C1) }P+\Pi\Lm\succ 0,\quad
\text{(C2) }\Rc{1},\Rc{2}\succeq 0,\quad
\text{(C3) }\Xibf\prec 0.
\end{equation}
Then, for $w\equiv 0$, every solution of \eqref{eq:model} exists on
$[0,\infty)$ and is unique, $\sup_{t\ge 0}|y(t)|\le\beta\|\phi\|_{\infty}$
with $\beta$ independent of the delay functions, $y(t)\to 0$ as $t\to\infty$,
and the origin is the unique equilibrium. If (C1), (C2) and
$\tilde\Xibf(\gamma)\prec 0$ hold, then (C3) holds and, for zero history,
$\|z\|_{2}\le\gamma\|w\|_{2}$ for every $w\in L_{2}[0,\infty)$.
\end{theorem}

\begin{remark}
\label{rem:input}
For a constant input $w\equiv\bar w$, (C1)--(C3) imply a unique and
globally asymptotically stable equilibrium $y^{*}$. Evaluating $\Xibf$
at the vector whose $y$-type blocks equal $u-v$ and whose $g$-type blocks
equal $g(u)-g(v)$ shows that $\Gamma(u)=-Du+(A+B_{1}+B_{2})g(u)+\bar w$
satisfies $|\Gamma(u)-\Gamma(v)|\ge c|u-v|$ with some $c>0$. Hence
$\Gamma$ is a homeomorphism \citep{forti1995new}, and Theorem~\ref{thm:network} applies
to $y-y^{*}$, whose nonlinearity satisfies (H1) with the same bounds.
\end{remark}

\subsection{Size-uniform certificate}
\label{ssec:uniform}
The variables are now tied to the chain. Every module carries the
$2\times 2$ blocks $P_{t}$, $S_{k,t}$ and $R_{k,t}$, the matrix
$X_{k,t}\in\R^{4\times 4}$ formed by the four $2\times 2$ blocks of
$X_{k}$, and the diagonal entries of $\Pi_{t}$, $T_{j,t}$ and $T_{ab,t}$
of its type $t$. The network variables are assembled block-diagonally.
Then $P+\Pi\Lm$ and, after a permutation, each $\Rc{k}$ are block
diagonal over the modules, and (C1), (C2) and the sign conditions reduce
to the local conditions
\begin{equation}
\label{eq:local}
\begin{gathered}
P_{t}+\Pi_{t}\Lm\succ 0,\quad\Rc{k,t}\succeq 0,\quad S_{k,t}\succeq 0,
\quad R_{k,t}\succeq 0,\\
\Pi_{t}\succeq 0,\quad T_{j,t}\succeq 0,\quad T_{ab,t}\succeq 0,\qquad
t,k=1,2,
\end{gathered}
\end{equation}
where $\Rc{k,t}$ is $\Rc{k}$ built from $R_{k,t}$ and $X_{k,t}$.

A cell consists of a type-1 module and the following type-2 module. In
an open chain with $N=2K$ modules that starts with type 1, cell
$j\in\{1,\dots,K\}$ contains the modules $2j-1$ and $2j$, and
$\zeta_{j}\in\R^{48}$ collects their coordinates of $\xi$. The coupling
splits into a part inside the cell and parts to the next and to the
previous cell,
\begin{equation*}
B_{2,0}=\begin{bmatrix}0&F_{f}\\ F_{b}&0\end{bmatrix},\
B_{2,+}=\begin{bmatrix}0&0\\ F_{f}&0\end{bmatrix},\
B_{2,-}=\begin{bmatrix}0&F_{b}\\ 0&0\end{bmatrix},
\end{equation*}
so that the dynamics of cell $j$ read
$\dot y^{[j]}=a_{0}\zeta_{j}+a_{+}\zeta_{j+1}+a_{-}\zeta_{j-1}+w^{[j]}$
with
\begin{equation}
\label{eq:a}
\begin{aligned}
a_{0}&=-D_{c}e_{y,0}+A_{c}e_{g,0}+B_{1,c}e_{g,1}+B_{2,0}e_{g,2},\\
a_{\pm}&=B_{2,\pm}e_{g,2}.
\end{aligned}
\end{equation}
Here $D_{c}=\diag(D_{1},D_{2})$, $A_{c}$ and $B_{1,c}$ are formed
likewise, and the selectors act on the 48 coordinates of a cell. With
$L_{c}$, $W_{c}$ and $\tilde R_{c}$ the matrices $L$, $W$ and $\tilde R$
of one cell, the cell blocks are
\begin{equation}
\label{eq:cell}
\begin{aligned}
\Xibf_{0}&=L_{c}+\Sym(W_{c}a_{0})+\textstyle\sum_{q\in\{0,+,-\}}
a_{q}^{\tp}\tilde R_{c}a_{q},\\
\Xibf_{1}&=W_{c}a_{+}+a_{-}^{\tp}W_{c}^{\tp}+a_{0}^{\tp}\tilde R_{c}a_{+}
+a_{-}^{\tp}\tilde R_{c}a_{0},
\end{aligned}
\end{equation}
and their $L_{2}$ counterparts are
\begin{equation}
\label{eq:cell2}
\begin{aligned}
\tilde\Xibf_{0}&=\begin{bmatrix}\Xibf_{0}+e_{y,0}^{\tp}e_{y,0}&
W_{c}+a_{0}^{\tp}\tilde R_{c}\\ \star&\tilde R_{c}-\gamma^{2}I
\end{bmatrix},\\
\tilde\Xibf_{1}&=\begin{bmatrix}\Xibf_{1}&a_{-}^{\tp}\tilde R_{c}\\
\tilde R_{c}a_{+}&0\end{bmatrix}.
\end{aligned}
\end{equation}

\begin{lemma}
\label{lem:cell}
In cell order, the matrix $\Xibf$ of an open chain with $N=2K$ modules
that starts with type 1 is
\begin{equation}
\label{eq:open}
\Xibf_{N}=\Tc_{K}(\Xibf_{0},\Xibf_{1})-A_{\mathrm{out}}^{\tp}
(I_{2}\otimes\tilde R_{c})A_{\mathrm{out}},
\end{equation}
where $A_{\mathrm{out}}=\col(a_{+}J_{1},a_{-}J_{K})$ and $J_{j}$ selects
$\zeta_{j}$. For a ring with $N=2K\ge 4$ modules,
$\Xibf_{N}=\mathcal{C}_{K}(\Xibf_{0},\Xibf_{1})$. The same identities hold
for $\tilde\Xibf_{N}(\gamma)$ with $(\tilde\Xibf_{0},\tilde\Xibf_{1})$ in
place of $(\Xibf_{0},\Xibf_{1})$ and $[a_{\pm}\ 0]$ in place of $a_{\pm}$.
\end{lemma}

\begin{theorem}
\label{thm:uniform}
Let the type variables satisfy \eqref{eq:local} and let $\epsilon>0$.

(a) If a symmetric $Y\in\R^{48\times 48}$ satisfies
\begin{equation}
\label{eq:gram}
M=\begin{bmatrix}Y&-\Xibf_{1}\\ -\Xibf_{1}^{\tp}&-\Xibf_{0}-I-Y
\end{bmatrix}\succeq\epsilon I,
\end{equation}
then every network in $\mathcal{N}$ satisfies (C1)--(C3) with
$\Xibf_{N}\preceq-(1+2\epsilon)I$.

(b) If a symmetric $\tilde Y\in\R^{52\times 52}$ satisfies
\begin{equation}
\label{eq:gramL2}
\tilde M=\begin{bmatrix}\tilde Y&-\tilde\Xibf_{1}\\
-\tilde\Xibf_{1}^{\tp}&-\tilde\Xibf_{0}-\tilde Y\end{bmatrix}
\succeq\epsilon I,
\end{equation}
then every network in $\mathcal{N}$ satisfies (C1), (C2) and
$\tilde\Xibf_{N}(\gamma)\preceq-2\epsilon I$.

In both cases, Theorem~\ref{thm:network} applies to every network in
$\mathcal{N}$, in case (b) with the same $\gamma$.
\end{theorem}

The conditions of Theorem~\ref{thm:uniform} have a fixed size: two sets
of type variables, one matrix $Y$ and, besides \eqref{eq:local}, one LMI
of order 96 for \eqref{eq:gram} or 104 for \eqref{eq:gramL2}, whatever
$N$ is. The
identity in \eqref{eq:gram} only fixes the scale, since (C1)--(C3) are
homogeneous in the variables.

\begin{lemma}
\label{lem:exact}
Fix the type variables, and let $H_{0}=-\Xibf_{0}-I$ and
$H_{1}=-\Xibf_{1}$. The following statements are equivalent.
(i)~$-\Xibf_{N}-I\succeq 0$ for every open chain with $N=2K$ modules that
starts with type 1. (ii)~$H_{0}+zH_{1}+\bar zH_{1}^{\tp}\succeq 0$ for
every $z\in\mathbb{C}$ with $|z|=1$. (iii)~There is a symmetric $Y$ with
$\bigl[\begin{smallmatrix}Y&H_{1}\\ H_{1}^{\tp}&H_{0}-Y\end{smallmatrix}
\bigr]\succeq 0$.
\end{lemma}

By Lemma~\ref{lem:exact}, whose proof removes the boundary term of
\eqref{eq:open} by test vectors that vanish on the end cells and then
applies the operator Fej\'er--Riesz theorem \citep{dritschel2010operator},
the cell LMI \eqref{eq:gram} adds no conservatism, apart from the margin
$\epsilon$, to the requirement that one set of tiled variables certify
every chain length. The conservatism
of Theorem~\ref{thm:uniform} lies in the functional \eqref{eq:lkf} and its
bounds, in the module-local and type-periodic variables and in the
independence from $N$. Section~\ref{sec:results} measures these
restrictions against certificates computed separately for each $N$.

\begin{remark}
\label{rem:modulewise}
With module-local variables, $V$ and the bound \eqref{eq:vdot} are sums
of module terms. The stability and $L_{2}$ conclusions of
Theorem~\ref{thm:uniform} therefore also hold when each module $i$ has
its own delays $\tau_{k}^{(i)}$ that satisfy (H2),
provided that each module emits its delayed signals with its own delays,
so that its neighbours receive $g(y^{(i)}(t-\tau_{2}^{(i)}(t)))$. The
vector $\xi$ is then formed with the delays of the respective module, and
the Wirtinger split and the sector terms are applied module by module.
\end{remark}

\begin{remark}
\label{rem:K4}
For $\taub_{1},\taub_{2}>0$, a variant replaces $y^{\tp}Py$ in
\eqref{eq:lkf} by
$\eta^{\tp}\mathcal{P}\eta$, $\eta=\col(y,I_{1},I_{2})$,
$I_{k}=\int_{t-\taub_{k}}^{t}y(s)\,ds=\tau_{k}\nu_{k}+(\taub_{k}-\tau_{k})
\omega_{k}$. Its bound is affine in $(\tau_{1}(t),\tau_{2}(t))$ and is
imposed at the four vertices $\tau_{k}\in\{0,\taub_{k}\}$, one $Y$ each,
with (C1) replaced by $\mathcal{P}+\diag(\Pi\Lm,S_{1}/\taub_{1},
S_{2}/\taub_{2})\succ 0$. The vertex argument needs delays common to all
modules, so the variant serves only for comparisons in
Section~\ref{sec:results}.
\end{remark}

\subsection{Certified regions}
\label{ssec:regions}
For the couplings $(\kappa\hat G,-\kappa\hat G^{\tp})$,
$\Ac=\Ac^{0}+\kappa\Ac^{c}$ is affine in $\kappa$, where
$\Ac^{c}=\hat B_{2}e_{g,2}$ and $\hat B_{2}$ is $B_{2}$ at $\kappa=1$,
whereas $L$ and $W$ do not depend on $\kappa$. For fixed variables,
$\theta\in[0,1]$ and $\kappa_{\theta}=(1-\theta)\kappa_{-}+\theta\kappa_{+}$,
\begin{equation}
\label{eq:convex}
\begin{aligned}
&(1-\theta)\Xibf(\kappa_{-})+\theta\Xibf(\kappa_{+})-\Xibf(\kappa_{\theta})\\
&\qquad=\theta(1-\theta)(\kappa_{+}-\kappa_{-})^{2}\Ac^{c\tp}\tilde R\Ac^{c}
\succeq 0,
\end{aligned}
\end{equation}
and the same identity holds for $\tilde\Xibf$ with $[\Ac^{c}\ 0]$ in place
of $\Ac^{c}$. Delay bounds are scaled along a fixed direction,
$\taub=s\taub^{0}$ with $s\ge 0$.

\begin{corollary}
\label{cor:region}
Let the type variables satisfy \eqref{eq:local}, let $\taub=s\taub^{0}$,
and let \eqref{eq:gram} hold at $\kappa_{-}$ and at $\kappa_{+}$ with
separate matrices $Y_{-}$ and $Y_{+}$. Then the conclusions of
Theorem~\ref{thm:network} hold for every network in $\mathcal{N}$, every
constant $\kappa\in[\kappa_{-},\kappa_{+}]$ and all delays that satisfy
(H2) with $\taub=s'\taub^{0}$, $0\le s'\le s$. The analogous statement
holds for \eqref{eq:gramL2}.
\end{corollary}

\begin{proof}
Theorem~\ref{thm:uniform} gives $\Xibf_{N}(\kappa_{\pm})\preceq
-(1+2\epsilon)I$, and \eqref{eq:convex} gives $\Xibf_{N}(\kappa_{\theta})
\preceq(1-\theta)\Xibf_{N}(\kappa_{-})+\theta\Xibf_{N}(\kappa_{+})$. The
local conditions do not depend on $\kappa$, and every delay allowed by
$s'$ is allowed by $s$.
\end{proof}

A verified pair $(\kappa_{-},\kappa_{+})$ at the delay scale $s$ thus
certifies the rectangle $[\kappa_{-},\kappa_{+}]\times[0,s]$ of the
$(\kappa,s)$ plane for all networks in $\mathcal{N}$, where the coupling
strength is an unknown constant common to all neighbour pairs.

\subsection{Computation and exact verification}
\label{ssec:compute}
The LMIs are formulated in CVXPY \citep{diamond2016cvxpy} and solved with
Clarabel \citep{goulart2026clarabel} using its direct solver qdldl, one
thread and chordal decomposition. Each solve runs in a separate process
with limits on wall time (1800\,s, or 3600\,s for some solves of
Remark~\ref{rem:K4}), resident memory (12\,GiB) and address space
(16\,GiB). A solve that is
stopped or fails counts as a failure, never as infeasibility. The conditions \eqref{eq:local} are imposed as
$P_{t}+\Pi_{t}\Lm\succeq c_{1}I$, $R_{k,t}\succeq 10^{-6}I$ and
$T_{j,t}\succeq 10^{-6}I$, with $c_{1}=1$ for \eqref{eq:gram},
$c_{1}=10^{-3}$ for \eqref{eq:gramL2} and $\epsilon=10^{-3}$ in both. A
candidate counts as verified when every condition holds with a positive
margin in double-precision eigenvalue computations.

Delay scales and coupling bounds are maximized by bisection, which tests
the upper end first, records the last verified and the first unverified
value, and stops after a solve that reached its limit. For
\eqref{eq:gramL2}, $\gamma^{2}$ is minimized first, and
$\hat\gamma^{2}=1.02\,\gamma^{2}_{\mathrm{opt}}$ is rounded up to three
significant digits (factors $1.05$ and $1.10$ follow if verification
fails). The LMIs are then solved again at fixed $\hat\gamma^{2}$ while a
common margin $t\in[0,1]$, added to every constraint, is maximized.

Every reported certificate of Theorem~\ref{thm:uniform} is, in addition,
verified in exact rational arithmetic. Each
stored double is replaced by its exact rational value, and $\kappa$,
$\taub$, $\hat\gamma^{2}$ and the margins enter as rationals. The sector
is exactly $[-1/10,11/10]$, which the double-precision values of $-0.1$
and $1.1$ used by the solver enclose. The cell blocks
\eqref{eq:cell}--\eqref{eq:cell2} are rebuilt with the same formulas. The
diagonal multipliers are checked by their signs. Every other local matrix
and $M-\epsilon I$ or $\tilde M-\epsilon I$ are factorized as $LDL^{\tp}$
without pivoting, and positive pivots prove positive definiteness. The
solver thus only proposes candidates, and whether a certificate of
Theorem~\ref{thm:uniform} holds is decided by this exact check alone,
independently of the solver tolerances. If this check fails at the largest verified point, the point is solved again with
the margin maximization, and if that fails too, the next smaller verified
point is checked. A separate check recomputes all factorizations and
rejects copies of a certificate in which one entry of $\Pi$, $Y$ or
$\hat G$, or one setting, has been altered. A certificate thus holds for
the network whose data are the exact values of the stored doubles.

For a fixed $N$, Theorem~\ref{thm:network} is also solved directly with
unrestricted variables (central) and with module-local variables that are
not tiled (block-central), under the margins that \eqref{eq:gram} and
\eqref{eq:gramL2} imply, namely $P+\Pi\Lm\succeq I$ and
$\Xibf_{N}\preceq-(1+2\cdot 10^{-3})I$, or $P+\Pi\Lm\succeq 10^{-3}I$ and
$\tilde\Xibf_{N}\preceq-2\cdot 10^{-3}I$. A third variant imposes the
block-central LMI through an explicit decomposition along its chordal
sparsity pattern, whose cliques span at most three neighbouring modules.
These per-$N$ certificates are verified in double precision. A
block-central certificate can also be verified exactly by factorizing
$-\Xibf_{N}-(1+2\cdot 10^{-3})I$ in module order, where the fill stays
within the band of neighbouring modules. As a check that bypasses
Lemma~\ref{lem:cell}, the tiled matrices $\Xibf_{N}$ and
$\tilde\Xibf_{N}(\hat\gamma)$ are assembled for open chains with $N\le 9$
and both starts and for rings with $N\le 10$, and their largest
eigenvalues are computed. Reported times are medians of five sequential
runs of the solver process, including the problem construction, on an
otherwise idle machine (Intel Core Ultra 7 270K Plus, 31\,GiB, Python
3.13, CVXPY 1.9.2, Clarabel 0.11.1).

\subsection{Assessment of conservatism}
\label{ssec:assess}
Outer bounds on the certifiable region and lower bounds on the $L_{2}$
gain use linear members $g(y)=Sy$ of (H1) with diagonal $S$. On rings,
$S$ is cell-periodic, and the four slopes of a cell form one of the $16$
vertex patterns of $\{\lmin,\lmax\}$, rounded into the sector, are all
$1/2$ (the slope of the SiLU at the origin) or are drawn at random. For
constant delays, the discrete Fourier transform over the cells decouples
a ring of $K$ cells into the cell symbols
\begin{equation}
\label{eq:symbol}
\begin{aligned}
\Delta_{\theta}(\lambda)={}&\lambda I+D_{c}-A_{c}S_{c}-B_{1,c}S_{c}
e^{-\lambda\tau_{1}}\\
&-\bigl(B_{2,0}+B_{2,+}e^{\mathrm{i}\theta}+B_{2,-}e^{-\mathrm{i}\theta}
\bigr)S_{c}e^{-\lambda\tau_{2}},
\end{aligned}
\end{equation}
$\theta=2\pi m/K$, $m=0,\dots,K-1$. A member with a root of
$\det\Delta_{\theta}(\lambda)=0$ in $\operatorname{Re}\lambda>0$ is
unstable, so no certificate exists for a class that contains it.
Rightmost roots are computed by pseudospectral collocation of the
infinitesimal generator \citep{breda2005pseudospectral} on 33 Chebyshev
nodes, refined by Newton's method and accepted at a relative residual of
at most $10^{-10}$. The outer bound $\kappa_{\mathrm{out}}(s)$ is the
smallest $\kappa$ at which an unstable member is found, for delays at the
four corners and the centre of the box
$[0,s\taub_{1}^{0}]\times[0,s\taub_{2}^{0}]$ and $64$ phases. Its root is
rechecked on 49 nodes, and at $1.25\,\kappa_{\mathrm{out}}(s)$ the growth
rate of a simulated ring agrees with $\operatorname{Re}\lambda$ within
5\%. Let $\kappa_{\mathrm{safe}}(s)$ be the supremum of all $b$ such that
every member with delay scale $s$ and $\kappa\in[0,b]$ is stable. An
unstable member is also a member for every larger delay scale, so
$\kappa_{\mathrm{safe}}(s')\le\kappa_{\mathrm{out}}(s)$ for $s'\ge s$,
which gives the outer staircase in Fig.~\ref{fig:region}(a), and every
interval $[0,\kappa]$ certified at the scale $s'$ has
$\kappa\le\kappa_{\mathrm{safe}}(s')$. The bound
$\kappa_{\mathrm{SiLU}}(s)$ uses $S=I/2$ only. Delays that switch
periodically between $0$ and their bounds along trapezoidal waveforms are
probed through the growth of the period map, estimated by power
iteration. The delay-independent M-matrix test \citep{zhang2003globally}
certifies every $\kappa<\kappa_{M}$, where $\rho\bigl(\lbar D_{c}^{-1}
(|A_{c}|+|B_{1,c}|+\kappa_{M}\sum_{q}|\hat B_{2,q}|)\bigr)=1$, $|\cdot|$
acts entrywise and $\hat B_{2,q}$ is $B_{2,q}$ at $\kappa=1$. Its Perron
vector is cell-periodic and certifies every network in $\mathcal{N}$.
Finally, $\gamma_{\mathrm{lb}}$ is the largest $\sigma_{\max}$ of the
inverse characteristic matrix at $\lambda=\mathrm{i}\omega$ over members
of chains and rings with constant delays and a grid of frequencies, and
$\gamma_{\mathrm{lb}}^{\mathrm{tv}}$ is the largest ratio $\|z\|/\|w\|$ on
$[0,60]$ for ring members with switching or constant delays, zero
history and sinusoidal inputs. Both bound the worst-case $L_{2}$ gain of
the class from below. Rounding the end slopes into the sector changes
$\gamma_{\mathrm{lb}}$ by less than $10^{-14}$ relative, and halving the
step changes the maximizing ratios of $\gamma_{\mathrm{lb}}^{\mathrm{tv}}$
by less than $10^{-7}$. Unlike the certificates, the outer and lower
bounds are computed in floating point.

\begin{figure*}[!t]
\centering
\incfig[\textwidth]{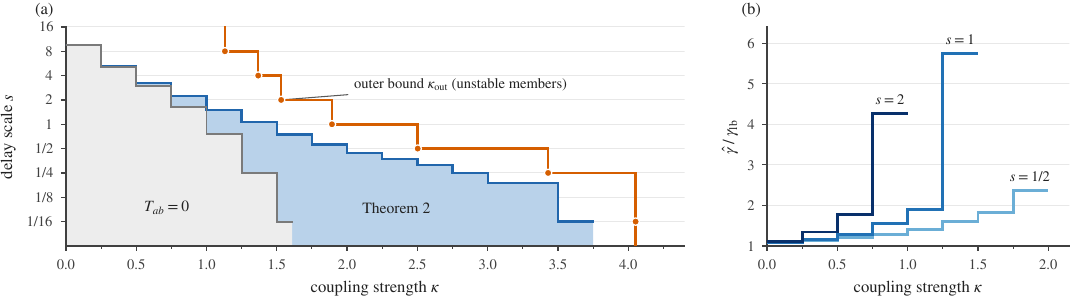}{5cm}{figures/fig\_region.pdf}
\caption{(a) Certified region in the $(\kappa,s)$ plane for
$(F_{f},F_{b})=(\kappa\hat G,-\kappa\hat G^{\tp})$. The shaded set is the
union of the exactly verified rectangles $[\kappa_{-},\kappa_{+}]\times[0,s]$
of Corollary~\ref{cor:region} and holds for every network in
$\mathcal{N}$, and its gray part is certified also for $T_{ab}=0$.
Markers give the couplings at which unstable members are found for
constant delays ($\kappa_{\mathrm{out}}$), and the outer staircase through
them bounds $\kappa_{\mathrm{safe}}$ from above
(Section~\ref{ssec:assess}). (b) Ratio of the certified gain $\hat\gamma$
to the lower bound $\gamma_{\mathrm{lb}}$ on the coupling cells of (a) for
three delay scales.}
\label{fig:region}
\end{figure*}

\subsection{Comparison and simulation protocol}
\label{ssec:protocol}
For a common delay $\tau_{1}=\tau_{2}=d(t)\in[0,\bar d]$, the delayed
terms of \eqref{eq:model} merge into $(B_{1}+B_{2})g(y(t-d(t)))$. For a
fixed $N$, Theorem~\ref{thm:network} is then applied with this single
channel, that is, with the blocks of the second channel removed from
$\xi$. Theorem~\ref{thm:uniform} covers the subclass through
$\taub=(\bar d,\bar d)$. For the same subclass, we implement Theorem~1 of
\citet{zhai2025hierarchical} at order one with the lower delay bound zero.
The implementation reproduces Example~1 of \citet{zhai2025hierarchical}.
With the lower bound one, it is feasible for the upper bound $4.28$ and
infeasible for $4.31$, in line with the reported bound $4.2964$.

Further type data are generated with seeded random draws (seeds
2026101301, 2026101302 and 2026101303). The diagonal
entries of $D_{t}$ are uniform in $[1,6/5)$, and $A_{t}$, $B_{1,t}$ and
the unscaled coupling block $G_{0}$ of Section~\ref{sec:results} have
standard Gaussian entries and are rescaled to the spectral norms of their
reference values, while $c$ is kept. A draw is accepted if
\eqref{eq:gram} is verified for the uncoupled modules ($\kappa=0$) at
$s=1$, and otherwise the next draw is taken.

Nonlinear simulations use the classical fourth-order Runge--Kutta method
with linear interpolation of delayed values and constant histories. Each
module $i$ has its own delays
$\tau_{k}^{(i)}(t)=\taub_{k}\bigl(5+3\sin(\omega_{ki}t+\varphi_{ki})\bigr)/8$
as in Remark~\ref{rem:modulewise}, with $\omega_{ki}$ and $\varphi_{ki}$
uniform in $[2,12]$ and $[0,2\pi)$. For zero history,
$\|y\|_{2}/\|w\|_{2}$ is evaluated on $[0,60]$. Each input component is
then a normalized sum of eight sinusoids with log-uniform frequencies in
$[0.05,5]$, windowed by $\sin^{2}(\pi t/20)$ on $[0,20]$ and scaled by
$a\in\{10^{-2},1,10\}$, or a sinusoid of amplitude $10^{-2}$ at the
frequency and direction where the frequency response of the member with
$S=I/2$ peaks. For constant inputs, the equilibrium of
Remark~\ref{rem:input} is computed by Newton's method and approached from
random constant histories. The step is $5\cdot 10^{-3}$, halved to check
the gain estimates, and $10^{-2}$ for the runs toward equilibria.

\section{Numerical results}
\label{sec:results}

\subsection{Data and settings}
\label{ssec:setup}
The type data and the coupling block are
\begin{equation*}
\begin{aligned}
D_{1}&=\diag(1,1.2), & D_{2}&=\diag(1.1,1),\\[3pt]
A_{1}&=\begin{bmatrix}0.1&0.3\\-0.3&0.1\end{bmatrix}, &
A_{2}&=\begin{bmatrix}0&0.3\\-0.4&0.1\end{bmatrix},\\[3pt]
B_{1,1}&=\begin{bmatrix}-0.25&0.05\\0&-0.2\end{bmatrix}, &
B_{1,2}&=\begin{bmatrix}-0.2&0\\0.05&-0.25\end{bmatrix},\\[3pt]
\hat G&=cG_{0}, & G_{0}&=\begin{bmatrix}0.6&0.2\\-0.1&0.5\end{bmatrix},
\end{aligned}
\end{equation*}
with $c=0.6918332448642871$, all taken as double-precision numbers. The
factor $c$ only sets the unit of $\kappa$ and places $\kappa=1$ at about
three times the M-matrix bound $\kappa_{M}$ of Section~\ref{ssec:assess}.
With
$(F_{f},F_{b})=(\kappa\hat G,-\kappa\hat G^{\tp})$, the matrix $B_{2}$ is
skew-symmetric for every network in $\mathcal{N}$, whereas the pattern
$(\hat G,-\hat G)$ has no such structure. The delay bounds are
$\taub=s\taub^{0}$ with $\taub^{0}=(1/4,1/2)$, and the sector is
$[-1/10,11/10]$ (Lemma~\ref{lem:silu}). The $L_{2}$ gain is bounded at
$s=1$ for three instances. I1 has $\kappa=1$, I2 the pattern
$(\hat G,-\hat G)$, and I3 every $\kappa\in[9/10,11/10]$, which
Corollary~\ref{cor:region} covers through \eqref{eq:gramL2} at the two
ends.

The largest delay scale at which the bisection below verifies
Theorem~\ref{thm:uniform} is denoted $s^{\star}(\kappa)$, and
$s^{\star}_{N}(\kappa)$ is its
counterpart for the per-$N$ certificates of Section~\ref{ssec:compute}
on open chains that start with type~1. Delay scales are bisected on
$[0,16]$, or on $[s^{\star},16]$ for $s^{\star}_{N}$, and delay bounds
on $[0,8]$, each with eight steps after the upper end, which resolves
$s^{\star}$ to $1/16$ and delay bounds to $1/32$. On the coupling axis,
\eqref{eq:gram} is solved at $\kappa_{i}=i/4$ for $s=1/16$ as long as it
is verified, and each pair $(\kappa_{i},\kappa_{i+1})$ is then solved with
common type variables (Corollary~\ref{cor:region}).

\subsection{Certified region and its tightness}
\label{ssec:region}
Fig.~\ref{fig:region}(a) shows the certified part of the $(\kappa,s)$ plane
as the union of the rectangles of Corollary~\ref{cor:region}, each
verified exactly for all networks in $\mathcal{N}$ with time-varying,
rate-free delays. Uncoupled
modules are certified up to the search limit $s=16$. As $\kappa$ grows,
$s^{\star}$ decreases from $77/8$ at $\kappa=1/4$ to $9/4$ at $\kappa=1$
and $1/16$ at $\kappa=15/4$.

With constant delays, unstable members are found at
$\kappa_{\mathrm{out}}(s)=4.052$, $3.430$, $2.502$ and $1.893$ for
$s=1/16$, $1/4$, $1/2$ and $1$. The staircase through these markers bounds
$\kappa_{\mathrm{safe}}(s)$ from above (Section~\ref{ssec:assess}) and
hence the certified set from outside. At these scales, the certified
coupling reaches $15/4$, $3$, $2$ and $3/2$, that is, 92\%, 87\%, 79\% and
79\% of $\kappa_{\mathrm{out}}(s)$ and thus at least these shares of
$\kappa_{\mathrm{safe}}(s)$, which the grid step $1/4$ makes conservative.
For $s=2$, $4$ and $8$, they fall to 65\%, 36\% and 22\%. The certificate is thus close to the
outer bound for short delays, and the gap widens as the delays grow.
Periodically switching delays produced no unstable member in the scanned
ranges, which extend to at least 97\% of $\kappa_{\mathrm{out}}(s)$ for
$s=1$, $2$ and $4$, so these waveforms do not account for the gap at
large $s$.

The linearization of the SiLU network at the origin ($S=I/2$) loses
stability only at $\kappa_{\mathrm{SiLU}}(s)=7.33$, $5.39$, $2.93$ and
$2.78$ for $s=1/4$, $1/2$, $4$ and $8$, beyond twice
$\kappa_{\mathrm{out}}(s)$, and nowhere in the scanned ranges for the
other scales. Stability at the origin is therefore a poor guide to the
coupling that the sector class admits. The delay-independent M-matrix
test certifies only $\kappa<\kappa_{M}=0.328$, less than one eleventh of
the coupling certified at $s=1/16$, and it enlarges the certified set
only for $s>83/16$, where the rectangles end at $\kappa=1/4$.

Without the difference terms ($T_{ab}=0$), the certified set shrinks to
the gray part of Fig.~\ref{fig:region}(a). The coupling certified at $s=1/16$ drops from
$15/4$ to $103/64$, and $s^{\star}(1)$ from $9/4$ to $13/8$. The loss
grows with $\kappa$, whereas both certificates reach the same coupling
on the grid at $s=4$ and $s=8$. The difference terms therefore matter
most for strong coupling.

\subsection{Independent delays}
\label{ssec:delays}
Fig.~\ref{fig:delay} shows, for $\kappa=1$, the largest certified bound
$\taub_{2}^{\star}$ on the transmission delay as a function of the bound
$\taub_{1}$ on the local delay. Each marker is an exactly verified
certificate of Theorem~\ref{thm:uniform} and covers the box between it
and the origin, since a delay bounded by $\taub'\le\taub$ is also bounded
by $\taub$. The two delays trade off mildly. As $\taub_{1}$ grows from
$0$ to $2$, $\taub_{2}^{\star}$ decreases from $45/32$ to $19/32$, so a
local delay bound eight times $\taub_{1}^{0}$ still admits 42\% of the
transmission delay bound certified for $\taub_{1}=0$. The difference
terms raise $\taub_{2}^{\star}$ by 37\% to 61\% across the whole
range. Because the two delays need not coincide,
the certified pairs, e.g., $(2,19/32)$ and $(0,45/32)$, reach well beyond
the common bound $\bar d^{\star}=15/16$ of Section~\ref{ssec:zhai}.

\subsection{Gain bounds}
\label{ssec:gains}
Table~\ref{tab:gain} lists the gain bounds at $s=1$.
Theorem~\ref{thm:uniform}(b) certifies $\hat\gamma=2.196$, $1.963$ and
$2.583$ for I1, I2 and I3, each for all networks in $\mathcal{N}$ and,
for I3, for every $\kappa\in[9/10,11/10]$. Since the worst-case gain of
the class is at least $\gamma_{\mathrm{lb}}^{\mathrm{tv}}$, $\hat\gamma$
exceeds it by at most 33\%, 29\% and 50\%. Without the difference terms,
$\hat\gamma$ is larger by factors of $3.1$, $4.0$ and $6.0$, whereas the
vertex variant of Remark~\ref{rem:K4} lowers it by only 4\% and 3\% for
I1 and I2.

\begin{figure}[!t]
\centering
\incfig[\columnwidth]{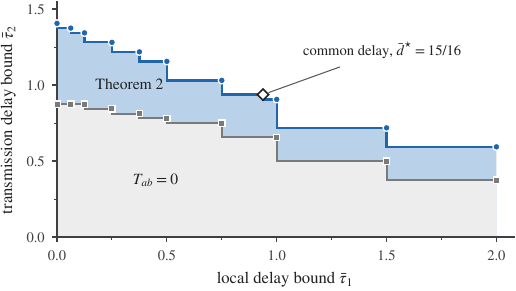}{4cm}{figures/fig\_delay.pdf}
\caption{Certified bounds $\taub_{1}$ on the local delay and $\taub_{2}$
on the transmission delay for $\kappa=1$. Each marker, including the
diamond for the common delay bound $\bar d^{\star}=15/16$ of
Section~\ref{ssec:zhai}, is an exactly verified certificate of
Theorem~\ref{thm:uniform} and covers the box between it and the origin.
The shaded set is the union of these boxes, and its gray part is certified
also for $T_{ab}=0$.}
\label{fig:delay}
\end{figure}

Fig.~\ref{fig:region}(b) compares $\hat\gamma$ with $\gamma_{\mathrm{lb}}$
on the coupling cells of Fig.~\ref{fig:region}(a), where each $\hat\gamma$
holds for the whole cell and for all smaller delay scales. For $\kappa\le 1/4$, $\hat\gamma$ exceeds
$\gamma_{\mathrm{lb}}$ by at most 11\% at all three delay scales. Toward the
boundary of the certified set, $\hat\gamma$ grows much faster than
$\gamma_{\mathrm{lb}}$, e.g., to $18.06$ against $3.14$ on $[5/4,3/2]$
at $s=1$.

\subsection{Price of uniformity and computational cost}
\label{ssec:uniformity}
\begin{figure*}[!t]
\centering
\incfig[\textwidth]{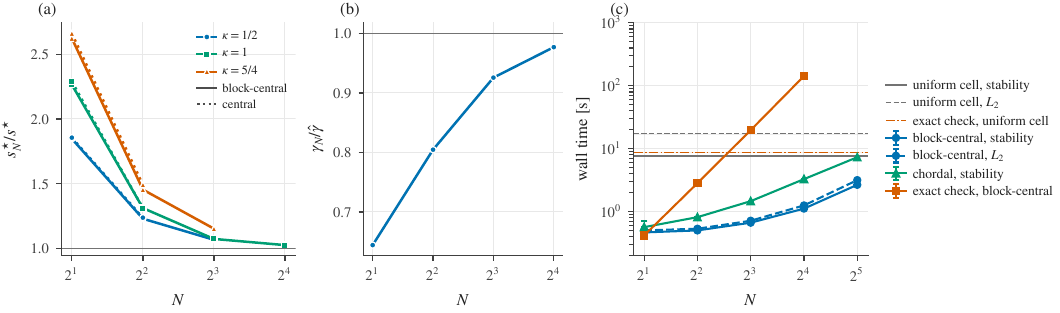}{5cm}{figures/fig\_scaling.pdf}
\caption{(a) Price of uniformity $s^{\star}_{N}/s^{\star}$ for
block-central and central certificates. (b) Block-central gain bound
$\gamma_{N}$ for I1 relative to $\hat\gamma$. (c) Wall times at
$\kappa=1$ and $s=1$, median and range of five runs. At $N=64$, the
construction of the chordal problem exceeds the address-space limit.}
\label{fig:scaling}
\end{figure*}

Fig.~\ref{fig:scaling}(a) compares $s^{\star}$ with $s^{\star}_{N}$. The
ratio $s^{\star}_{N}/s^{\star}$ is the price of certifying all networks
at once. It is largest for $N=2$, with $1.85$, $2.27$ and $2.62$ for
$\kappa=1/2$, $1$ and $5/4$, and decays quickly with $N$. For $\kappa=1$,
it is $1.31$ at $N=4$, $1.07$ at $N=8$ and $1.02$ at $N=16$, the last
within the bisection resolution, since $s^{\star}_{16}=2359/1024$ is below
the first unverified uniform scale $37/16$. The decay agrees with
Lemma~\ref{lem:cell}.
Only the end cells of an open chain receive the favourable term in
\eqref{eq:open}, whose influence fades as $N$ grows, and rings have no
such term. Central and block-central certificates, computed for $N=2$ and
$4$, differ by less than 3\%, so the module-local structure of the
variables costs little there, and the price of uniformity is confined to
short chains. The gain in
Fig.~\ref{fig:scaling}(b) follows the same trend, and the block-central
bound $\gamma_{N}$ for I1 reaches $0.977\hat\gamma$ at $N=16$.

The tiled certificates also satisfy the central and block-central
conditions for $N=2$, $4$, $8$ and $16$. For open chains with $N\le 9$ and both starts and for
rings with $N\le 10$, the assembled matrices have largest eigenvalues of
at most $-0.0040$ for I1 to I3, with I3 checked at five values of
$\kappa$. For the stability certificate at $\kappa=1$ and $s=9/4$, they
are at most $-12.06$. Both values lie below the guaranteed bounds
$-2\epsilon$ and $-(1+2\epsilon)$.

Fig.~\ref{fig:scaling}(c) reports wall times at $\kappa=1$ and $s=1$.
For every $N$, the cell problem has, besides the local conditions, 1298
scalar variables and one LMI of order 96 for stability and 1501 variables
and one LMI of order 104 for the gain. It takes 7.7\,s for stability,
17.3\,s for the minimization of $\gamma^{2}$ for I1 and 8.8\,s for the
exact check of the stability certificate. A block-central stability
solve is faster for a single network up to $N=32$, with 0.47 to 2.7\,s,
but it covers that network only, and its exact verification grows from
0.42\,s at $N=2$ to 143\,s at $N=16$. The
chordal variant takes 0.57 to 7.3\,s up to $N=32$ and fails at $N=64$,
where the problem construction exceeds the address-space limit.

\subsection{Common delay and comparison}
\label{ssec:zhai}
Table~\ref{tab:zhai} compares the largest certified common delay bound
$\bar d^{\star}$, where the values for one $N$ refer to open chains that
start with type~1. For $N=2$, the functional of
\citet{zhai2025hierarchical} certifies larger bounds than
Theorem~\ref{thm:network} with the merged channel, e.g., $11/2$ against
$25/8$ at $\kappa=1$, and the vertex variant of Remark~\ref{rem:K4}
reaches $137/32$. A certificate for one $N$ says nothing about longer
chains, however, and the bound of Theorem~\ref{thm:network} falls from
$25/8$ at $N=2$ to $7/4$ at $N=4$. The conditions of \citet{zhai2025hierarchical}
grow from 2544 scalar variables and an LMI of order 76 at $N=2$ to 155648
variables and order 608 at $N=16$, and for $N=4$ and $\kappa=1$ their
first solve did not finish within 1800\,s. Theorem~\ref{thm:uniform}
certifies $\bar d^{\star}=15/16$ at $\kappa=1$ for every $N$, both starts
and rings with a fixed problem size, and the same certificates admit
independent delays (Fig.~\ref{fig:delay}). The vertex variant of
Remark~\ref{rem:K4} raises this bound to $33/32$ with four LMIs of order
96 in place of one.

\begin{table}[!t]
\caption{$L_{2}$-gain bounds at $s=1$. Upper bounds are rounded up and lower
bounds down. The variant of Remark~\ref{rem:K4} was not computed for I3.}
\label{tab:gain}
\begin{tabular*}{\linewidth}{@{}l@{\extracolsep{\fill}}ccc@{}}
\toprule
 & I1 & I2 & I3\\
\midrule
$\hat\gamma^{2}$, Theorem~\ref{thm:uniform} & $241/50$ & $77/20$ & $667/100$\\
$\hat\gamma$, Theorem~\ref{thm:uniform} & $2.196$ & $1.963$ & $2.583$\\
$\hat\gamma$, Theorem~\ref{thm:uniform} with $T_{ab}=0$ & $6.849$ & $7.881$ & $15.621$\\
$\hat\gamma$, Remark~\ref{rem:K4} & $2.115$ & $1.908$ & --\\
$\gamma_{\mathrm{lb}}$ & $1.419$ & $1.346$ & $1.586$\\
$\gamma_{\mathrm{lb}}^{\mathrm{tv}}$ & $1.653$ & $1.525$ & $1.732$\\
$\hat\gamma/\gamma_{\mathrm{lb}}^{\mathrm{tv}}$ & $1.33$ & $1.29$ & $1.50$\\
\bottomrule
\end{tabular*}
\end{table}

\begin{table}[!t]
\caption{Certified common delay bound $\bar d^{\star}$ and problem size. Values for one $N$ are verified in double precision and values for
all $N$ exactly. The bound $8$ is the search limit. The variant of
Remark~\ref{rem:K4} for all $N$ has four LMIs of order 96.}
\label{tab:zhai}
\footnotesize\setlength{\tabcolsep}{0pt}
\begin{tabular*}{\linewidth}{@{}l@{\extracolsep{\fill}}cccccc@{}}
\toprule
 & \citet{zhai2025hierarchical} & \multicolumn{2}{c}{Thm.~\ref{thm:network}} & Rem.~\ref{rem:K4} & Thm.~\ref{thm:uniform} & Rem.~\ref{rem:K4}\\
$N$ & $2$ & $2$ & $4$ & $2$ & all & all\\
\midrule
$\kappa=1/2$ & $\ge 8$ & $21/4$ & $61/16$ & $231/32$ & $17/8$ & $81/32$\\
$\kappa=1$ & $11/2$ & $25/8$ & $7/4$ & $137/32$ & $15/16$ & $33/32$\\
$\kappa=5/4$ & $63/16$ & $39/16$ & $5/4$ & $53/16$ & $21/32$ & $11/16$\\
\midrule
variables & 2544 & 110 & 396 & 136 & 1298 & 4862\\
LMI order & 76 & 28 & 56 & 28 & 96 & 96\\
\bottomrule
\end{tabular*}
\end{table}

\subsection{Further data and simulations}
\label{ssec:further}
For the three further type-data sets, the first draw of each seed was
accepted. Their uniform delay scales $s^{\star}(1/2)$ and $s^{\star}(1)$
are $113/16$ and $3$, $7$ and $3$, and $89/16$ and $27/16$, with
$\hat\gamma=2.210$, $1.563$ and $5.469$ at $\kappa=1$ and $s=1$, all
verified exactly. Block-central certificates for $N=16$ exceed
$s^{\star}(1)$ by only 1.7\%, 3.4\% and 6.7\%, so the small price of
uniformity for long chains is not specific to the reference data.

Simulations at $\kappa=1$ and $s=1$ illustrate Remarks~\ref{rem:input}
and~\ref{rem:modulewise} with heterogeneous activations. An open chain of
$N=9$ modules that starts with type~2 and a ring of $N=10$ modules use
SiLU, $1.1\tanh$ and the piecewise-linear function with slopes $11/10$
and $-1/10$ in turn over the modules. For zero input, eight random
constant histories per network decay below $2.1\cdot 10^{-39}$ by
$t=100$. The largest gain
estimate over all inputs is $0.946$ for the chain and $0.947$ for the
ring, below $\hat\gamma=2.196$ of I1, and halving the step changes the
estimates by less than $10^{-7}$ relative to their values. For constant
inputs, chains with $N=8$ and $9$ SiLU modules were run with three random
inputs and 50 random histories each. Every trajectory reached the Newton
equilibrium within $2.6\cdot 10^{-14}$ by $t=400$, consistent with the
uniqueness and attractivity stated in Remark~\ref{rem:input}.

\section{Conclusion}
\label{sec:conclusion}
This paper has shown that stability and an $L_{2}$-gain bound of chains of
delayed neural modules can be certified for all lengths by conditions of
fixed size. Fixed integration limits free the analysis from the rates of
change of the delays, and module-local variables that repeat with the module
type reduce every chain and ring to the blocks of one cell. By the operator
Fej\'er--Riesz theorem, the reduction of the stability certificate to one
cell adds no conservatism of its own, apart from a margin. The certificates
cover all activations in the slope sector, module-wise delays
(Remark~\ref{rem:modulewise}) and constant inputs, and they extend to
certified regions of coupling strength and delay scale.

The numerical study shows that length independence costs little for long
chains, since certificates of the same functional computed for a single
length gain mainly for short chains. For short delays, the certified coupling lies close to the
outer bound set by unstable members of the model class, and for weak
coupling the certified gains lie close to lower bounds on the worst-case
gain. Tests that ignore the delay bounds or rely on the linearization at the
origin misjudge the admissible coupling by large factors. The sector terms
on activation differences are essential for strong coupling, since removing
them more than halves the coupling certified for short delays. Because the
cell problem has a fixed size and is verified exactly once for all lengths,
certification becomes one problem instead of a sequence of problems that
grow with the length.

\appendix
\section{Proofs}
\label{app:proofs}

\begin{proof}[Proof of Lemma~\ref{lem:vdot}]
The first two terms of \eqref{eq:lkf} have the derivative
$2(y^{\tp}P+g_{0}^{\tp}\Pi)\dot y=2\xi^{\tp}W\dot y$, and the third has
the derivative $\sum_{k}(y^{\tp}S_{k}y-\bar y_{k}^{\tp}S_{k}\bar y_{k})$.
The fourth term equals $\sum_{k}\taub_{k}\int_{t-\taub_{k}}^{t}
(s-t+\taub_{k})\dot y^{\tp}(s)R_{k}\dot y(s)\,ds$ and has the derivative
$\dot y^{\tp}\tilde R\dot y-\sum_{k}\taub_{k}\mathcal{J}_{k}$ with
$\mathcal{J}_{k}=\int_{t-\taub_{k}}^{t}\dot y^{\tp}R_{k}\dot y\,ds$. If
$\taub_{k}=0$, then $\mathcal{J}_{k}=0$ and $\Fc{k}\xi=0$. Otherwise, let
$\alpha=\tau_{k}(t)/\taub_{k}$ and $\col(\chi_{1},\chi_{2})=\Fc{k}\xi$.
For $\alpha\in(0,1)$, the Wirtinger inequality \citep{seuret2013wirtinger}
on $[t-\tau_{k}(t),t]$ and on $[t-\taub_{k},t-\tau_{k}(t)]$ gives
$\taub_{k}\mathcal{J}_{k}\ge\alpha^{-1}\chi_{1}^{\tp}\hat R_{k}\chi_{1}
+(1-\alpha)^{-1}\chi_{2}^{\tp}\hat R_{k}\chi_{2}$, and the reciprocally
convex bound \citep{park2011reciprocally} with $\Rc{k}\succeq 0$ gives
$\taub_{k}\mathcal{J}_{k}\ge\xi^{\tp}\Fc{k}^{\tp}\Rc{k}\Fc{k}\xi$. For
$\alpha\in\{0,1\}$, one of $\chi_{1},\chi_{2}$ vanishes, and the Wirtinger
inequality on the other interval gives the same bound. Only the value
$\tau_{k}(t)$ enters. Adding the terms of \eqref{eq:sector}, which are
nonnegative by (H1), gives \eqref{eq:vdot}. Substituting
$\dot y=\Ac\xi+w$ and $z=e_{y,0}\xi$ yields the two consequences.
\end{proof}

\begin{proof}[Proof of Theorem~\ref{thm:network}]
For $w=0$, the right-hand side $F(t,\psi)=-D\psi(0)+Ag(\psi(0))+
\sum_{k}B_{k}g(\psi(-\tau_{k}(t)))$ of \eqref{eq:model} is continuous in
$t$ by (H2) and globally Lipschitz in $\psi\in C([-h,0],\R^{n})$ with the
constant $L_{F}=\|D\|+\lbar(\|A\|+\|B_{1}\|+\|B_{2}\|)$, and
$F(t,0)=0$. Successive approximations give a unique global solution,
which is $C^{1}$ on $(0,\infty)$ and satisfies
$\|y_{t}\|_{\infty}\le e^{L_{F}t}\|\phi\|_{\infty}$ and
$|\dot y(t)|\le L_{F}\|y_{t}\|_{\infty}$ \citep{hale2013introduction}. With
an additive $w\in L_{2}$, the same argument gives an absolutely
continuous solution.

Choose $\lambda>0$ with $\Xibf\preceq-\lambda I$. For $t\ge h$, $y$ is
Lipschitz on $[t-h,t]$, so $V$ is absolutely continuous on $[h,\infty)$,
and Lemma~\ref{lem:vdot} gives $\dot V\le-\lambda|\xi|^{2}\le
-\lambda|y(t)|^{2}$ almost everywhere. By (H1), $g_{i}(r)/r\in[\lmin,\lmax]$
for $r\ne 0$, so $2\int_{0}^{y_{i}}g_{i}(s)\,ds\ge\lmin y_{i}^{2}$ and
$V(t)\ge y^{\tp}(P+\Pi\Lm)y\ge c|y(t)|^{2}$ with
$c=\lambda_{\min}(P+\Pi\Lm)>0$. The bounds on $[0,h]$ give
$V(h)\le\beta_{0}\|\phi\|_{\infty}^{2}$, where $\beta_{0}$ depends only on
$h$, $L_{F}$ and the variables. Hence $|y(t)|^{2}\le\beta_{0}
\|\phi\|_{\infty}^{2}/c$ for $t\ge h$, since $V$ is nonincreasing there, and together with the
bound on $[0,h]$ this proves the uniform bound. Moreover, $\int_{h}^{\infty}|y|^{2}\,dt\le V(h)/\lambda$,
and $y$ and $\dot y$ are bounded, so $|y|^{2}$ is uniformly continuous and
$y(t)\to 0$. An equilibrium is a constant solution and is therefore zero.

The leading block of $\tilde\Xibf(\gamma)\prec 0$ gives
$\Xibf\prec-e_{y,0}^{\tp}e_{y,0}\preceq 0$. For zero history, $y$ vanishes
on $[-h,0]$ and is absolutely continuous on $[-h,\infty)$, and
$\dot y=\Ac\xi+w$ is locally square integrable, so $V(0)=0$ and
Lemma~\ref{lem:vdot} holds for almost every $t\ge 0$. Then
$\dot V+|z|^{2}-\gamma^{2}|w|^{2}\le 0$, and integration over $[0,t_{1}]$
with $V(t_{1})\ge 0$ gives $\int_{0}^{t_{1}}|z|^{2}\,dt\le\gamma^{2}
\int_{0}^{t_{1}}|w|^{2}\,dt$ for every $t_{1}>0$.
\end{proof}

\begin{proof}[Proof of Lemma~\ref{lem:cell}]
With module-local variables, the three terms of \eqref{eq:vdot} split
into terms of single modules. Hence
$\xi^{\tp}\Xibf\xi=\sum_{j}(\zeta_{j}^{\tp}L_{c}\zeta_{j}+2\zeta_{j}^{\tp}
W_{c}\dot y^{[j]}+\dot y^{[j]\tp}\tilde R_{c}\dot y^{[j]})$ with
$\dot y^{[j]}=a_{0}\zeta_{j}+a_{+}\zeta_{j+1}+a_{-}\zeta_{j-1}$, where
$\zeta_{0}$ and $\zeta_{K+1}$ are absent for an open chain and indices
are taken modulo $K$ for a ring. Collecting terms gives $\Xibf_{0}$ in the
diagonal blocks and $\Xibf_{1}$ in the blocks $(j,j+1)$. A block
$(j-1,j+1)$ could only arise from $a_{-}^{\tp}\tilde R_{c}a_{+}$, which
vanishes because $a_{-}$ has nonzero rows only in the first module of a
cell, $a_{+}$ only in the second, and $\tilde R_{c}$ is block diagonal
over the modules. In an open chain, the first and the last cell lack the
terms $a_{+}^{\tp}\tilde R_{c}a_{+}$ and $a_{-}^{\tp}\tilde R_{c}a_{-}$
that the absent cells $0$ and $K+1$ would contribute, which gives
\eqref{eq:open}. The input $w^{[j]}$ enters only $\dot y^{[j]}$, which
gives the blocks of $\tilde\Xibf_{0}$ and $\tilde\Xibf_{1}$.
\end{proof}

\begin{proof}[Proof of Theorem~\ref{thm:uniform}]
(a) By \eqref{eq:local}, (C1) and (C2) hold for every network in
$\mathcal{N}$. Let $H_{0}=-\Xibf_{0}-I$ and $H_{1}=-\Xibf_{1}$. Then
\eqref{eq:gram} implies $Y\succeq\epsilon I$ and $H_{0}-Y\succeq\epsilon I$.
With $E_{j}$ the selector of the cell pair $(j,j+1)$,
\begin{equation*}
\begin{aligned}
\Tc_{K}(H_{0},H_{1})={}&\textstyle\sum_{j=1}^{K-1}E_{j}^{\tp}ME_{j}\\
&+J_{1}^{\tp}(H_{0}-Y)J_{1}+J_{K}^{\tp}YJ_{K}\succeq 2\epsilon I,
\end{aligned}
\end{equation*}
since every cell appears in exactly two terms (for $K=1$ the sum is
empty). By Lemma~\ref{lem:cell} and $\tilde R_{c}\succeq 0$,
$-\Xibf_{N}-I\succeq\Tc_{K}(H_{0},H_{1})\succeq 2\epsilon I$ for the open
chains with $N=2K$ modules that start with type 1. For a ring,
$\mathcal{C}_{K}(H_{0},H_{1})=\sum_{j=1}^{K}E_{j}^{\tp}ME_{j}\succeq
2\epsilon I$ with indices modulo $K$. Every other open chain is a
consecutive part of an open chain with $N_{b}=N+2$ (even $N$) or
$N_{b}=N+3$ (odd $N$) modules that starts with type 1 and contains both
neighbours of the part. Let $E$ select the coordinates of the part. The
local terms and $\Sym(W\Ac)$ restrict exactly, because the variables are
module-local and every neighbour pair carries $(F_{f},F_{b})$. Splitting
$\Ac^{\tp}\tilde R\Ac$ into the rows of the part and of its two
neighbours gives $\Xibf_{\mathrm{part}}=E^{\tp}\Xibf_{N_{b}}E-A_{r}^{\tp}
\tilde R_{r}A_{r}\preceq-(1+2\epsilon)I$, where $A_{r}$ collects the
neighbour rows of $\Ac E$ and $\tilde R_{r}\succeq 0$ their weights.
(b) The same steps with $\tilde H_{0}=-\tilde\Xibf_{0}$ and
$\tilde H_{1}=-\tilde\Xibf_{1}$ give $\tilde\Xibf_{N}(\gamma)\preceq
-2\epsilon I$, since $w$ enters only the rows of its own module.
\end{proof}

\begin{proof}[Proof of Lemma~\ref{lem:exact}]
(i)$\Rightarrow$(ii): Let $|z|=1$, $v\in\mathbb{C}^{48}$ and $K\ge 4$,
and let $x$ have the cell components $x_{1}=x_{K}=0$ and $x_{j}=z^{j}v$
for $1<j<K$. Then $A_{\mathrm{out}}x=0$, and (i), which also holds for
complex vectors, gives $0\le x^{*}\Tc_{K}(H_{0},H_{1})x=(K-3)v^{*}H(z)v
+v^{*}H_{0}v$ with $H(z)=H_{0}+zH_{1}+\bar zH_{1}^{\tp}$. Dividing by
$K-3$ and letting $K\to\infty$ gives $v^{*}H(z)v\ge 0$.
(ii)$\Rightarrow$(iii): By the operator Fej\'er--Riesz theorem
\citep{dritschel2010operator}, $H(z)=P(z)^{*}P(z)$ on $|z|=1$ for some
$P(z)=P_{0}+zP_{1}$ with $P_{0},P_{1}\in\mathbb{C}^{48\times 48}$.
Comparing coefficients gives $H_{0}=P_{0}^{*}P_{0}+P_{1}^{*}P_{1}$ and
$H_{1}=P_{0}^{*}P_{1}$. Hence the Gram matrix $[P_{0}\ P_{1}]^{*}
[P_{0}\ P_{1}]\succeq 0$ has the diagonal blocks $P_{0}^{*}P_{0}$ and
$H_{0}-P_{0}^{*}P_{0}$ and the off-diagonal block $H_{1}$. Its entrywise
real part is positive semidefinite as well and, since $H_{0}$ and $H_{1}$
are real, yields (iii) with $Y=\operatorname{Re}(P_{0}^{*}P_{0})$.
(iii)$\Rightarrow$(i): This is the telescoping step in the proof of
Theorem~\ref{thm:uniform} with $\epsilon=0$.
\end{proof}

\bibliographystyle{abbrvnat}
\bibliography{references}

\end{document}